\documentclass{article}
\usepackage{spconf}
\usepackage{hyperref}
\usepackage{xcolor}
\usepackage{amsthm,amsmath,amsfonts}
\usepackage{amssymb}
\usepackage{mathtools}
\usepackage{bm}
\usepackage{titlesec} 
\usepackage{cancel}
\usepackage{empheq}
\usepackage{algorithm}
\usepackage{algorithmic}

\usepackage[caption=false,font=footnotesize]{subfig}

\newcommand{\gradz}{\nabla_{\bm{\theta}}}

\theoremstyle{plain}
\newtheorem{theorem}{Theorem}[section]

\theoremstyle{definition}

\title{Online Gradient Computation for Warping Gaussian Process Transformations}
\name{
Emilio Ruiz-Moreno\textsuperscript{1,2}, Konstantinos Slavakis\textsuperscript{3}, and Baltasar Beferull-Lozano\textsuperscript{1,2}
\thanks{\textsuperscript{1}SIGIPRO Department, Simula Metropolitan Center for Digital Engineering, Oslo, Norway; \textsuperscript{2}SURE-AI Center, Simula Research Laboratory, Oslo, Norway; \textsuperscript{3}Institute of Science Tokyo, Yokohama, Japan. 
This work was supported by the DISCO grant 338740, the DRIVE grant 360486, and the SURE-AI Center grant 357482 from the Research Council of Norway.
\\ \textbf{code}: \href{https://github.com/SIGIPRO/onlineWGP}{https://github.com/SIGIPRO/onlineWGP} }
}
\address{ }

\begin{document}
\maketitle
\begin{abstract}
Warped Gaussian processes (GPs) handle non-Gaussian observations by mapping them into a latent standard GP via a parametric transformation called warping. 
Existing streaming variants, however, either optimize the warping parameters periodically or sacrifice analytical tractability for a higher model capacity.
To bridge this gap, we show that the gradient of the instantaneous negative log-likelihood of a warped GP admits an exact recursive computation.
Based on this result, we propose a novel online method for warped GPs that jointly updates the latent GP moments and optimizes the warping parameters.
\end{abstract}


\vspace{-8pt}
\section{Introduction}
\vspace{-4pt}

Gaussian processes (GPs) provide a non-parametric Bayesian framework for probabilistic function approximation \cite{rasmussen2003gaussian}, proving especially useful when the true underlying function is analytically unknown or expensive to query \cite{swiler2020survey,ruiz2025doubly}.
By placing a prior directly over the space of functions of interest, GPs offer a principled approach to uncertainty quantification, making them particularly attractive for applications where predictive confidence is as critical as point accuracy; for instance, in robotics \cite{deisenroth2011pilco} or geostatistics \cite{diggle1998model,dewey2026deep}. 
Crucially, the mathematical tractability of GPs allows for exact Bayesian inference, yielding closed-form predictive distributions when paired with Gaussian function observation likelihoods.

Despite their widespread success, standard GPs are fundamentally limited by their core assumption that the function observations are adequately modeled by a joint Gaussian distribution. 
Indeed, this assumption often breaks down in practice, as real-world observations may exhibit heavy tails, distinct skewness, or physical constraints (such as strict positivity). Consequently, applying standard GPs to such non-Gaussian observations can produce highly inaccurate predictions and poorly calibrated uncertainty bounds.

Motivated by these limitations, warped GPs \cite{snelson2003warped} extend standard GPs by applying a parametric transformation—termed warping—to non-Gaussian observations, mapping them into latent targets where standard GP assumptions hold. 
By jointly estimating latent GP moments and warping parameters, warped GPs accommodate complex observation likelihoods while retaining the analytical tractability of standard GPs.

However, the implementation of warped GPs in memory-constrained or real-time tasks remains a significant challenge. 
This is because, to the best of our knowledge, there is no online method to jointly update the latent GP moments and optimize the warping parameters of warped GPs.
The methods arguably closest to this work in the literature either optimize the warping parameters periodically \cite{kou2013sparse} or trade analytical tractability for higher model capacity \cite{bui2016deep}.

In this paper, we show that the gradient of the negative log-likelihood (NLL) of a warped GP can be computed recursively in a exact manner. 
Based on this, we propose a novel online method for warped GPs that jointly updates the latent GP moments and optimizes the warping parameters.

\vspace{-8pt}
\section{Background}
\label{sec:background}
\vspace{-4pt}

\begin{figure}[t]
    \centering
    \includegraphics[width=\linewidth]{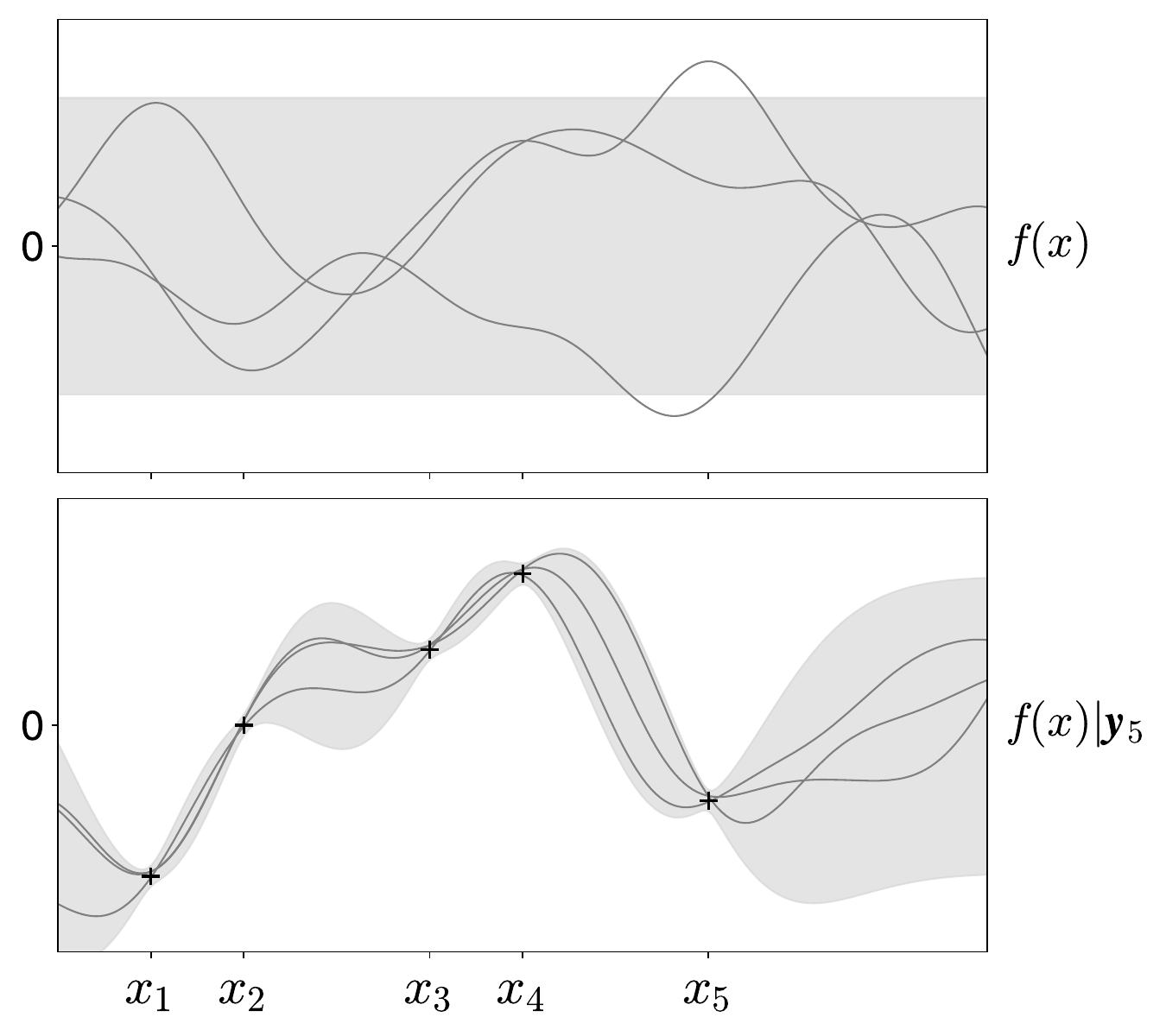}
    \caption{GP prior (top) and GP posterior (bottom), i.e., the GP prior conditioned on observations (indicated by + markers). The solid lines represent sample functions drawn from the GP. The shaded area covers $\pm 1.96$ standard deviations from the mean. Example adapted from \cite[Ch. 2]{rasmussen2003gaussian}.}
    \label{fig:GP_prior-posterior}
\end{figure}

Consider the problem of inferring an unknown, real-valued function $f$ over an arbitrary $d$-dimensional input space $\mathcal{X}\subseteq\mathbb{R}^d$, based on prior beliefs about its function space and a set of $n$ (possibly corrupted) function observations $y_1,\dots,y_n$ at corresponding input locations $\bm{x}_1,\dots,\bm{x}_n$.

Such a problem can be addressed analytically, provided certain assumptions hold, as we detail next.

\vspace{-4pt}
\subsection{Gaussian processes}
\label{ssec:GPs}

A GP model typically assumes the function $f$ is a realization of a zero-mean GP prior and the function observations are corrupted by white Gaussian noise.
That is, 
\begin{subequations}
\label{eq:gp_model}
\begin{align}
    f(\bm{x}) &\sim \mathcal{GP} \left( 0 , \kappa(\bm{x},\bm{x}') \right) , \label{seq:GP_prior} \\
    y_i &= f(\bm{x}_i) + \epsilon_i , \label{seq:measurement_model} 
\end{align}
\end{subequations}
where the kernel $\kappa:\mathcal{X}\times\mathcal{X}\to\mathbb{R}$ is the covariance function of the GP prior, and each $\epsilon_i \sim \mathcal{N}(0,\sigma^2)$ is an observation noise term with standard deviation $\sigma$.

Thanks to Gaussianity and the closure properties of GPs~\cite{rasmussen2003gaussian}, the posterior distribution over functions conditioned on a set of $n$ observations collected in the observation vector $\bm{y}_n = [y_1,\dots,y_n]^\top\in\mathbb{R}^n$ is also Gaussian.
Specifically,
\begin{equation}
\label{eq:conditional_GP}
    f(\bm{x})|\bm{y}_n \sim \mathcal{N} \left( m_n(\bm{x}) , v_n(\bm{x}) \right) ,
\end{equation}
with closed-form posterior mean and variance
\begin{subequations}
\label{eq:GP_moments}
\begin{align}
    m_n(\bm{x}) &= \bm{k}_n(\bm{x})^\top \bm{\alpha}_n , \\
    v_n(\bm{x}) &= \kappa(\bm{x},\bm{x}) - \bm{k}_n(\bm{x})^\top \bm{\Omega}_n  \bm{k}_n(\bm{x}) ,
\end{align}
\end{subequations}
respectively, where 
\begin{subequations}
\label{eq:GP_posterior_terms}
\begin{align}
    \bm{\alpha}_n &= \bm{\Omega}_n \bm{y}_n \in \mathbb{R}^n , \label{seq:alpha_vector} \\
    \bm{\Omega}_n &= \left( \bm{K}_n + \sigma^2\bm{I}_n \right)^{-1} \in \mathbb{S}^n_{\geq 0} , \label{seq:Omega_matrix} ,
\end{align}
\end{subequations}
with $\bm{k}_n(\bm{x})\in\mathbb{R}^n$ and $\bm{K}_n \in \mathbb{S}^n_{\geq 0}$ constructed as $[\bm{k}_n(\bm{x})]_i = \kappa(\bm{x}_i,\bm{x})$ and $[\bm{K}_n]_{i,j} = \kappa(\bm{x}_i,\bm{x}_j)$, for all $i,j \in \{1,\dots,n\}$.

To aid conceptual understanding, an illustration of the prior \eqref{seq:GP_prior} and posterior \eqref{eq:conditional_GP} of a GP is provided in Fig. \ref{fig:GP_prior-posterior}.

\subsubsection{Recursive Gaussian process updates}
\label{sssec:recursive}
Direct computation of the terms $\bm{\alpha}_n$ and $\bm{\Omega}_n$ for the posterior moments in \eqref{eq:GP_moments} scales poorly with the number of observations $n$, mostly due to the matrix inversion on the right-hand side of \eqref{seq:Omega_matrix}.
Fortunately, one can still compute these terms exactly, but at a reduced computational load, by exploiting the following recursive structure
\begin{subequations}
\begin{align}
    \bm{\alpha}_n &= \begin{bmatrix} \bm{\alpha}_{n-1} \\ 0 \end{bmatrix} - \frac{y_n - m_{n-1}(\bm{x}_n)}{s_n} \begin{bmatrix} \bm{\omega}_n \\ -1 \end{bmatrix} , \\
    \bm{\Omega}_n &= \begin{bmatrix} \bm{\Omega}_{n-1} & \bm{0}_{n-1} \\ \bm{0}_{n-1}^\top & 0 \end{bmatrix} + \frac{1}{s_n} \begin{bmatrix} \bm{\omega}_n \\ -1 \end{bmatrix} \begin{bmatrix} \bm{\omega}_n^\top & -1 \end{bmatrix} , \label{seq:recursive_Omega}
\end{align}
\end{subequations}
where $s_n = v_{n-1}(\bm{x}_n) + \sigma^2 \in \mathbb{R}$, and $\bm{\omega}_n = \bm{\Omega}_{n-1}\bm{k}_{n-1}(\bm{x}_n) \in \mathbb{R}^{n-1}$.
Thus, the GP posterior \eqref{eq:conditional_GP} can be updated recursively as each new observation $y_n$ and associated input location $\bm{x}_n$ become available, simply by maintaining $\bm{\alpha}_n$ and $\bm{\Omega}_n$ as state variables.

\subsubsection{Sparse Gaussian process updates}
\label{sssec:sparsification}
Indeed, the recursive updates in Sec. \ref{sssec:recursive} effectively avoid the computational complexity bottleneck of an $n\times n$ matrix inversion.
However, computing the posterior moments in \eqref{eq:GP_moments} still requires evaluating an $n \times 1$ vector dot product for the mean and an $n\times n$ quadratic form for the variance.
That is, the computational and memory requirements scale as $\mathcal{O}(n)$ for the mean, and $\mathcal{O}(n^2)$ for the variance.

This unbounded growth with the number of observations $n$, is known in the literature as the ``curse of kernelization'' \cite{wang2012breaking,calandriello2017efficient} and is typically addressed via kernel approximation methods such as Nyström \cite{williams2000using} and random Fourier features \cite{rahimi2007random,money2023sparse}, or via dictionary-based sparsification methods including forgetting mechanisms \cite{slavakis2014online} and the approximate linear dependency (ALD) technique \cite{engel2004kernel,engel2005algorithms}. 

In this work, we rely on the ALD technique, as it naturally integrates with the recursive updates in Sec. \ref{sssec:recursive} (see Sec. \ref{sec:method} for more details).

\vspace{-4pt}
\subsection{Warped Gaussian processes}
\label{ssec:warped_GPs}

Warped GP models assume that observations follow a GP model only after applying a monotonic, parametric, and differentiable transformation referred to as warping \cite{snelson2003warped}.
Specifically, given a warping transformation $g$ of $r$ parameters $ \bm{\theta} \in \Theta \subseteq \mathbb{R}^r$, the observation model in \eqref{seq:measurement_model} becomes
\begin{equation}
\label{eq:latent_target}
    z_i := g(y_i ; \bm{\theta}) = f(\bm{x}_i) + \epsilon_i ,
\end{equation}
where each $z_i\in\mathbb{R}$ serves as a latent target.
Accordingly, the term $\bm{\alpha}_n$, defined in \eqref{seq:alpha_vector}, becomes $\bm{\alpha}_n = \bm{\Omega}_n \bm{z}_n \in \mathbb{R}^n$, with $\bm{z}_n = [z_1,\dots,z_n]^\top\in\mathbb{R}^n$.

Warped GPs can be seen as a generalization of GPs. 
In fact, they are typically non-Gaussian and even asymmetric in the observation space.

By invoking the change of variables formula~\cite[Ch. 2]{casella2024statistical}, the joint probability density $p$ of any $n$ observations $\bm{y}_n$ can be expressed in terms of the corresponding density $q$ of latent targets $\bm{z}_n$, and the warping transformation $g$ and warping parameters $\bm{\theta}$.
Specifically,
\begin{equation}
\label{eq:p_y}
p( \bm{y}_n | \bm{\theta} ) = q( \bm{g}(\bm{y}_n ; \bm{\theta}) ) \prod_{i=1}^n \frac{\partial g(y_i ; \bm{\theta})}{\partial y_i} ,
\end{equation}
where $\bm{g}$ applies the warping transformation $g$ elementwise, and the joint probability of latent targets follows
\begin{equation}
\label{eq:p_z}
    q(\bm{z}) = \left( (2\pi)^n \left| \bm{\Omega}_n^{-1} \right| \right)^{-\frac{1}{2}} \exp \left( -\frac{1}{2} \bm{z}_n^\top \bm{\Omega}_n \bm{z}_n \right) 
\end{equation}
by construction from \eqref{eq:latent_target}.

One of the greatest strengths of warped GPs is that the warping parameters $\bm{\theta}$ can be estimated directly from the observations $\bm{y}_n$ without requiring explicit knowledge of their joint probability density $p$. 
As long as the latent density $q$ and the warping transformation $g$ are defined, \eqref{eq:p_y} provides a tractable likelihood that can be optimized.

\begin{algorithm}[!t]
\caption{Proposed online method for warped GPs.}
\label{alg}
\begin{algorithmic}[1] 
\STATE \textit{\% tilde superscripts distinguish sparsified variables from their exact counterparts.}
\STATE \textbf{Choose} $g$, $\Theta$, $\kappa$, $\{\sigma_n\}$, the ALD sparsification threshold $\nu$, and the projected optimizer $\Pi_{\Theta}$.
\STATE \textbf{Initialize} $\bm{\theta}_0 \in \Theta$, $\tilde{\bm{K}}^{-1}_0 = 1/\kappa(\bm{x}_1,\bm{x}_1)$, $\tilde{\bm{\alpha}}_0 = \tilde{\bm{C}}_0 = 0$, $\tilde{\bm{B}}_0 = \bm{0}_r$, and the dictionary $\mathcal{D}_0 = \{\bm{x}_1\}$.
\FOR{$n=1,2,\dots$}
    \STATE \textbf{Observe} $\bm{x}_n$ and $y_n$
    \STATE \textbf{Compute} $\hat{\bm{a}}_n = \tilde{\bm{K}}^{-1}_{n-1}\tilde{\bm{k}}_{n-1}(\bm{x}_t)$, and
    \STATE $\delta_n = \kappa(\bm{x}_n,\bm{x}_n) - \tilde{\bm{k}}_{n-1}(\bm{x}_n)^\top \hat{\bm{a}}_n$
    \STATE \textbf{Warp} $z_n = g(y_n;\bm{\theta}_{n-1})$ and $\dot{\bm{z}}_n = \nabla_{\bm{\theta}} \, g(y_n;\bm{\theta}_{n-1})$
    \STATE \textbf{Compute} $\tilde{e}_n = z_n - \tilde{\bm{k}}_{n-1}(\bm{x}_n)^\top \tilde{\bm{\alpha}}_{n-1}$, and
    \STATE $\tilde{\bm{b}}_n = \dot{\bm{z}}_n - \tilde{\bm{B}}_{n-1}\tilde{\bm{k}}_{n-1}(\bm{x}_n)$
    \IF{ $\delta_n > \nu$ }
    \STATE \textbf{Update} 
    $ \tilde{\bm{K}}^{-1}_n = \frac{1}{\delta_n} \begin{bmatrix} \delta_n \tilde{\bm{K}}^{-1}_{n-1} + \hat{\bm{a}}_n\hat{\bm{a}}_n^\top & -\hat{\bm{a}}_n \\ -\hat{\bm{a}}_n^\top & 1 \end{bmatrix} \nonumber $
    \STATE \textbf{Compute} 
    $\tilde{\bm{c}}_n = \begin{bmatrix} - \tilde{\bm{C}}_{n-1}\tilde{\bm{k}}_{n-1}(\bm{x}_n) \\ 1 \end{bmatrix}$, and
    \STATE $\tilde{s}_n \! = \! \kappa(\bm{x}_n,\bm{x}_n) \! + \! \sigma^2_n \! - \! \tilde{\bm{k}}_{n-1}(\bm{x}_n)^\top \tilde{\bm{C}}_{n-1}\tilde{\bm{k}}_{n-1}(\bm{x}_n)$
    \STATE \textbf{Reshape} $\tilde{\bm{C}}_{n-1} = \begin{bmatrix} \tilde{\bm{C}}_{n-1} & \bm{0}_{D_n-1} \\ \bm{0}_{D_n-1}^\top & 0 \end{bmatrix}$, 
    \STATE $\tilde{\bm{B}}_{n-1} = [\tilde{\bm{B}}_{n-1},\bm{0}_r]$, and $\tilde{\bm{\alpha}}_{n-1} = \begin{bmatrix} \tilde{\bm{\alpha}}_{n-1} \\ 0 \end{bmatrix}$
    \STATE \textbf{Update} $\mathcal{D}_n = \mathcal{D}_{n-1}\cup\{\bm{x}_n\}$
    \ELSE 
    \STATE \textbf{Update} $\tilde{\bm{K}}_n = \tilde{\bm{K}}_{n-1}$
    \STATE \textbf{Compute} $\tilde{\bm{c}}_n = \hat{\bm{a}}_n - \tilde{\bm{C}}_{n-1}\tilde{\bm{k}}_{n-1}(\bm{x}_n)$, and $\tilde{s}_n =$
    \STATE $\tilde{\bm{k}}_{n-1}(\bm{x}_n)^\top\hat{\bm{a}}_n \! + \! \sigma_n^2 \! - \! \tilde{\bm{k}}_{n-1}(\bm{x}_n)^\top \tilde{\bm{C}}_{n-1}\tilde{\bm{k}}_{n-1}(\bm{x}_n)$
    \STATE \textbf{Update} $\mathcal{D}_n = \mathcal{D}_{n-1}$
    \ENDIF
    \STATE \textbf{Update} $\tilde{\bm{\alpha}}_n \! = \! \tilde{\bm{\alpha}}_{n-1} \! + \! \frac{\tilde{e}_n}{\tilde{s}_n} \tilde{\bm{c}}_n$, $\tilde{\bm{B}}_n \! = \! \tilde{\bm{B}}_{n-1} \! + \! \frac{1}{\tilde{s}_n} \tilde{\bm{b}}_n \tilde{\bm{c}}_n^\top$,
    \STATE and $\tilde{\bm{C}}_n = \tilde{\bm{C}}_{n-1} + \frac{1}{\tilde{s}_n} \tilde{\bm{c}}_n \tilde{\bm{c}}_n^\top$
    \STATE \textbf{Compute} $\nabla_{\bm{\theta}} \tilde{\ell}_n(\bm{\theta}_{n-1}) = \frac{\tilde{e}_n}{\tilde{s}_n}\tilde{\bm{b}}_n - \nabla_{\bm{\theta}} \log ( \frac{\partial z_n}{\partial y_n} )$
    \STATE \textbf{Update} $\bm{\theta}_n = \Pi_\Theta ( \bm{\theta}_{n-1} , \nabla_{\bm{\theta}} \tilde{\ell}_n(\bm{\theta}_{n-1}) )$
\ENDFOR
\STATE \textbf{Return} $\mathcal{D}_n$, $\tilde{\bm{\alpha}}_n$, $\tilde{\bm{C}}_n$, and $\bm{\theta}_n$
\end{algorithmic}
\end{algorithm}
\begin{figure*}[t]
    \centering
    \subfloat[True process, warped GP, and distribution slices (gray dashed line)\label{subfig:process}]{%
        \includegraphics[width=0.63\textwidth]{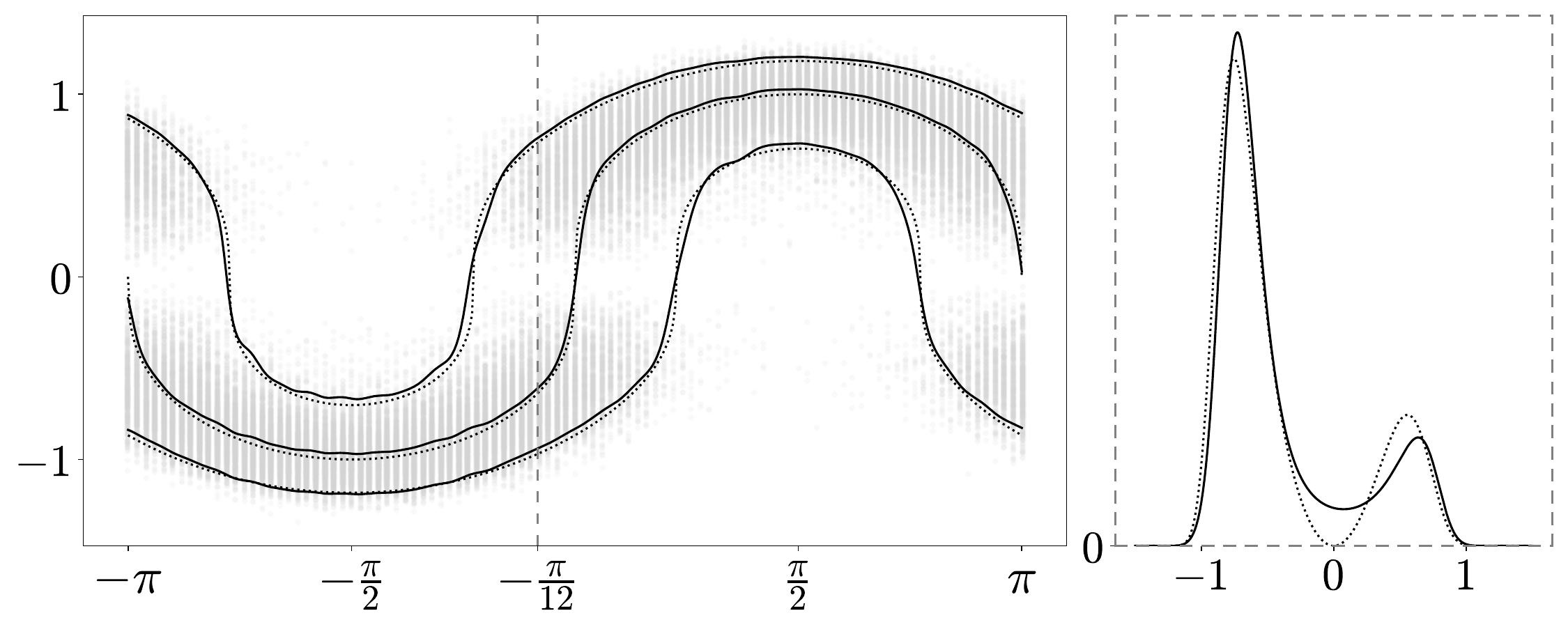}%
    }\hfill
    \subfloat[Forward and inverse warping transformation\label{subfig:warping}]{%
        \includegraphics[width=0.34\textwidth]{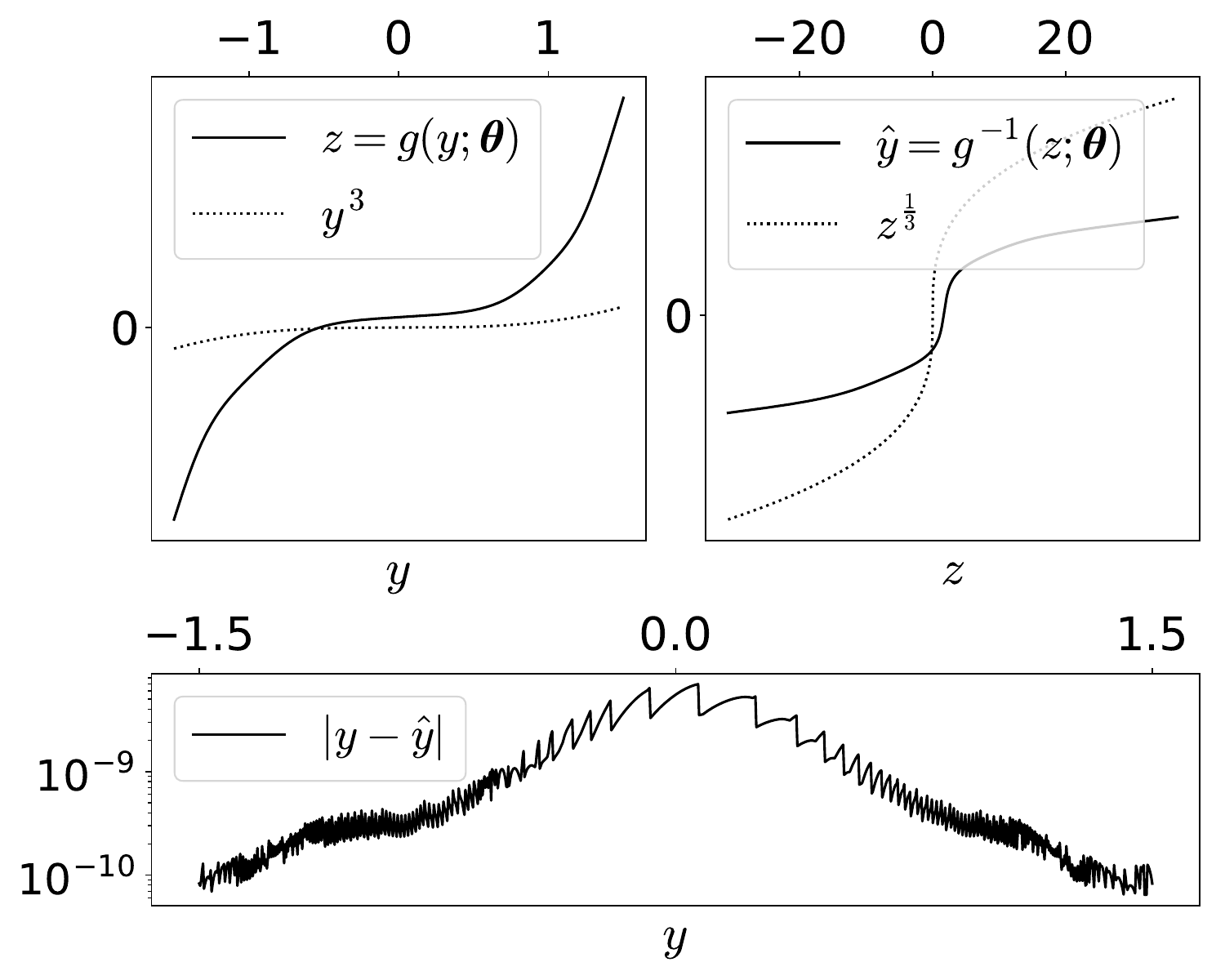}%
    }
    \caption{In Fig. \ref{subfig:process}, the observations are marked with gray dots. The dotted lines show the true generating distribution $p(\bm{y})$ and the solid lines show the warped GP prediction $p(\bm{y}|\bm{\theta})$. The triplets of lines represent the median, along with the $2.5$th and $97.5$th percentiles in each case. The cross-section shows the probability densities at $x=-\frac{\pi}{12}$; that is, $p(y|x=-\frac{\pi}{12}$) and $p(y|\bm{\theta},x=-\frac{\pi}{12})$. Fig. \ref{subfig:warping} shows the forward-inverse pass of the estimated warping transformation as well as its inversion error.}
    \label{fig:experiment}
\end{figure*}

\section{Recursive gradient computation}
\label{sec:recursive_gradient_computation}
\vspace{-4pt}

The $n$-observation joint NLL
\begin{equation}
\label{eq:nll}
    L_n(\bm{\theta}) = - \log p(\bm{y}_n|\bm{\theta}) 
\end{equation}
measures the discrepancy between the observations $\bm{y}_n$ and the joint probability density $p$ (implicitly) described by the warping parameters $\bm{\theta}$.
Using the multiplication rule of probability \cite{ross2014first}, one can readily get that 
\begin{equation}
\label{eq:recursive_nll}
    L_n(\bm{\theta}) = L_{n-1}(\bm{\theta}) + \ell_n(\bm{\theta}) ,
\end{equation}
where the instantaneous NLL 
\begin{equation}
\label{eq:inll}
    \ell_n(\bm{\theta}) = -\log p(y_n | \bm{y}_{n-1},\bm{\theta}) ,
\end{equation}
quantifies how unexpected the last observation $y_n$ was given the previous observations $\bm{y}_{n-1}$.

In many applications, the set of observations $y_1,\dots,y_n$ is too large to fit in memory, or is streamed, rendering standard batch processing intractable.
As a result, estimating the warping parameters by directly minimizing the NLL in \eqref{eq:nll} with respect to $\bm{\theta}$ may not be possible in practice.
On the other hand, by exploiting the additive decomposition of the NLL in \eqref{eq:recursive_nll}, one can naturally transition from a batch to an online estimation framework.
Specifically, the warping parameter estimate can be iteratively refined utilizing the gradient of the instantaneous NLL, e.g., by a first-order method \cite{beck2017first}.

Although evaluating every instantaneous NLL $\ell_n(\bm{\theta})$ depends on the previous observations $\bm{y}_{n-1}$, Theorem \ref{thm:recursive} demonstrates that its gradient can be computed recursively. 

\begin{theorem}[Recursive gradient computation]
\label{thm:recursive}
The gradient of the instantaneous NLL in \eqref{eq:inll} with respect to the warping parameters $\bm{\theta}$ can be computed as 
\begin{equation}
\label{eq:gradient_inll_rule}
    \gradz \ell_n(\bm{\theta}) = \frac{e_n}{s_n}\bm{b}_n - \gradz \log \left( \frac{\partial z_n}{\partial y_n} \right) ,
\end{equation}
where the terms $e_n\in\mathbb{R}$ and $\bm{b}_n \in \mathbb{R}^r$ correspond to
\begin{subequations}
\begin{align}
    e_n &= \bm{z}_n - m_{n-1}(\bm{x}_n) , \text{ and } \\
    \bm{b}_n &= \dot{\bm{z}}_n - \bm{B}_{n-1} \bm{k}_{n-1}(\bm{x}_n) ,    
\end{align} 
\end{subequations}
with $\dot{\bm{z}}_n = \gradz g(y_n;\bm{\theta}) \in \mathbb{R}^r$, and the term $\bm{B}_n \in \mathbb{R}^{r\times n}$ acts as a state variable updated according to
\begin{equation}
    \bm{B}_n = \begin{bmatrix} \bm{B}_{n-1} & \bm{0}_r \end{bmatrix} - \frac{\dot{\bm{z}}_n - \bm{B}_{n-1}\bm{k}_{n-1}(\bm{x}_n)}{s_n} \begin{bmatrix} \bm{\omega}_n^\top & -1 \end{bmatrix} .
\end{equation}
\end{theorem}
\begin{proof}
See the supplementary material \ref{sup:proof}.
\end{proof}

As a result, the gradient of the instantaneous NLL $\nabla_{\bm{\theta}}\ell_n(\bm{\theta})$ can be updated recursively as each new observation $y_n$ and associated input location $\bm{x}_n$ become available by maintaining $\bm{\alpha}_n$, $\bm{\Omega}_n$ (as in Sec. \ref{sssec:recursive}), and $\bm{B}_n$ as state variables.

\section{Proposed online method}
\label{sec:method}
\vspace{-4pt}

Algorithm \ref{alg} outlines the proposed online method for warped GPs, which jointly updates the latent GP moments and optimizes the warping parameters.
It integrates the recursive GP updates from Sec. \ref{sssec:recursive}, the recursive gradient computation of the instantaneous NLL from Sec. \ref{sec:recursive_gradient_computation}, and the ALD sparsification technique introduced in Sec. \ref{sssec:sparsification}. 
As a result, the update rule \eqref{eq:gradient_inll_rule} yields an approximate gradient.

The per-step computational and memory complexity is $\mathcal{O}(D_n^2 + r D_n)$.
We refer the reader to the supplementary material \ref{sup:ald_sparsification} and \ref{sup:sparse-recursive_gp_updates} for further details.

\vspace{-8pt}
\section{Revisiting the 1D regression task}
\label{sec:experiment}
\vspace{-4pt}

To evaluate our proposed method, we adapt the simple 1D regression task introduced in \cite[Sec. 4]{snelson2003warped} to an online setting.
This task is specifically designed to generate non-Gaussian observations with sharp transitions, causing standard GPs to fail; therefore, relying entirely on the modeling capabilities of the warped GP.

\textbf{Experimental setup}. 
The input domain is the real subset $\mathcal{X} = [-\pi,\pi] \subset \mathbb{R}$. 
The observational data are generated from a sinusoidal signal corrupted by white Gaussian noise followed by a cubic root transformation. 
That is, every $n$th observation is generated as $y_n = (\sin(x_n) + \varepsilon_n)^{\frac{1}{3}}$ with $\varepsilon_n \sim \mathcal{N}(0,\frac{1}{3^2})$.
Then, we generate $1000$ realizations of $101$ uniformly spaced observations each.
That is, a total of $101000$ observations, input location pairs that are fed sequentially to the warped GP model.

\textbf{Model configuration}.
The choice of warping transformation follows that of the revisited task.
That is, $g(y;\bm{\theta}) = y + \sum^t_{i=1} \theta_{1,i} \tanh ( \theta_{2,i} ( y + \theta_{3,i}) )$,  where $\bm{\theta} = [\bm{\theta}_1^\top,\bm{\theta}_2^\top,\bm{\theta}_3^\top]^\top\in\Theta$, and $\Theta = \{ \bm{\theta} \in \mathbb{R}^{3t} : \bm{\theta}_1,\bm{\theta}_2 \succeq \bm{0}_{t}, \bm{1}_t^\top\bm{\theta}_1 > 0, \text{ and } \bm{1}_t^\top\bm{\theta}_2 > 0 \}$, to ensure monotonicity.
We use $t=10$ terms, which leads to a total of $r=30$ warping parameters.
Similarly, the choice of the covariance function remains the same, i.e., $\kappa(x,x') = k_a \exp( - \frac{1}{2k_w^2}(x-x')^2 )$, with $k_a = 2$, and $k_w = 2\cdot\frac{2\pi}{101} \approx 0.124$. 
Lastly, the noise standard deviation of the latent observation model is set to $\sigma = 3$.

\textbf{Implementation details}.
We use the Adam\footnote{Pytorch \cite{paszke2019pytorch} default configuration.} optimizer \cite{kingma2014adam}, followed by a projection onto the feasible set $\Theta$.
We set the ALD threshold to $\nu = 0.1$.

\textbf{Results}.
The results of the experiment are summarized in Fig. \ref{fig:experiment}.
Fig. \ref{subfig:process} shows that the predictions of the warped GP closely align with the true generating distribution.
On the other hand, Fig. \ref{subfig:warping} shows that the learned warped transformation closely resembles the true generating cubic transformations in the region of interest where most observations $y$ lie, i.e., $-1 \lesssim y \lesssim 1$. 
It also shows that its inversion is numerically stable.
The (ALD sparsification) dictionary ends up with $68$ atoms.
Finally, the warped GP achieves an empirical NLL of $-0.28$ nats per observation.
This is a $0.76$ nat improvement over a standard GP under an identical model configuration except for a well-calibrated noise variance.

\vspace{-8pt}
\section{Conclusion and future work}
\label{sec:conclusion}
\vspace{-4pt}

We presented a recursive gradient evaluation for the instantaneous NLL of warped GPs, requiring only one extra state variable over standard recursive implementations. 
Building on this result, we introduced an online algorithm for jointly updating latent GP moments and optimizing warping parameters, demonstrating its performance on a warped GP benchmark. 
Future work includes online learning of the kernel parameters (or directly a suitable kernel \cite{ruiz2023online}) and the observation model noise variance; evaluations on broader benchmarks are underway.

\small
\bibliographystyle{IEEEbib}
\bibliography{references.bib}

@incollection{rasmussen2003gaussian,
  title={Gaussian processes in machine learning},
  author={Rasmussen, Carl Edward},
  booktitle={Summer school on machine learning},
  pages={63--71},
  year={2003},
  publisher={Springer}
}

@book{casella2024statistical,
  title={Statistical inference},
  author={Casella, George and Berger, Roger},
  year={2024},
  publisher={CRC press}
}

@article{snelson2003warped,
  title={Warped gaussian processes},
  author={Snelson, Edward and Ghahramani, Zoubin and Rasmussen, Carl},
  journal={Advances in neural information processing systems},
  volume={16},
  year={2003}
}

@book{horn2012matrix,
  title={Matrix analysis},
  author={Horn, Roger A and Johnson, Charles R},
  year={2012},
  publisher={Cambridge university press}
}

@article{diggle1998model,
  title={Model-based geostatistics},
  author={Diggle, Peter J and Tawn, Jonathan A and Moyeed, Rana A},
  journal={Journal of the Royal Statistical Society Series C: Applied Statistics},
  volume={47},
  number={3},
  pages={299--350},
  year={1998},
  publisher={Oxford University Press}
}

@techreport{dewey2026deep,
  title={Deep-AeroGP: deep kernel learning for projecting the regional climate response to anthropogenic aerosol emission changes},
  author={Dewey, Maura and Wilcox, Laura and Samset, Bj{\o}rn and Ekman, Annica},
  year={2026},
  institution={Copernicus Meetings}
}

@inproceedings{deisenroth2011pilco,
  title={PILCO: A model-based and data-efficient approach to policy search},
  author={Deisenroth, Marc and Rasmussen, Carl E},
  booktitle={Proceedings of the 28th International Conference on machine learning (ICML-11)},
  pages={465--472},
  year={2011}
}

@article{swiler2020survey,
  title={A survey of constrained Gaussian process regression: Approaches and implementation challenges},
  author={Swiler, Laura P and Gulian, Mamikon and Frankel, Ari L and Safta, Cosmin and Jakeman, John D},
  journal={Journal of Machine Learning for Modeling and Computing},
  volume={1},
  number={2},
  year={2020},
  publisher={Begel House Inc.}
}

@inproceedings{ruiz2025doubly,
  title={Doubly Truncated Mode Kriging},
  author={Ruiz-Moreno, Emilio and Beferull-Lozano, Baltasar},
  booktitle={2025 IEEE Statistical Signal Processing Workshop (SSP)},
  pages={306--310},
  year={2025},
  organization={IEEE}
}

@article{engel2004kernel,
  title={The kernel recursive least-squares algorithm},
  author={Engel, Yaakov and Mannor, Shie and Meir, Ron},
  journal={IEEE Transactions on signal processing},
  volume={52},
  number={8},
  pages={2275--2285},
  year={2004},
  publisher={IEEE}
}

@book{engel2005algorithms,
  title={Algorithms and representations for reinforcement learning},
  author={Engel, Yaakov},
  year={2005},
  publisher={Hebrew University of Jerusalem Jerusalem}
}

@book{ross2014first,
  title={A first course in probability},
  author={Ross, Sheldon M and Ross, Sheldon M and Ross, Sheldon M and Ross, Sheldon M and Math{\'e}maticien, Etats-Unis},
  volume={8},
  year={2014},
  publisher={Pearson London}
}

@book{beck2017first,
  title={First-order methods in optimization},
  author={Beck, Amir},
  year={2017},
  publisher={SIAM}
}

@article{wang2012breaking,
  title={Breaking the curse of kernelization: Budgeted stochastic gradient descent for large-scale svm training},
  author={Wang, Zhuang and Crammer, Koby and Vucetic, Slobodan},
  journal={The Journal of Machine Learning Research},
  volume={13},
  number={1},
  pages={3103--3131},
  year={2012},
  publisher={JMLR. org}
}

@article{calandriello2017efficient,
  title={Efficient second-order online kernel learning with adaptive embedding},
  author={Calandriello, Daniele and Lazaric, Alessandro and Valko, Michal},
  journal={Advances in Neural Information Processing Systems},
  volume={30},
  year={2017}
}

@article{rahimi2007random,
  title={Random features for large-scale kernel machines},
  author={Rahimi, Ali and Recht, Benjamin},
  journal={Advances in neural information processing systems},
  volume={20},
  year={2007}
}

@article{williams2000using,
  title={Using the Nystr{\"o}m method to speed up kernel machines},
  author={Williams, Christopher and Seeger, Matthias},
  journal={Advances in neural information processing systems},
  volume={13},
  year={2000}
}

@article{money2023sparse,
  title={Sparse online learning with kernels using random features for estimating nonlinear dynamic graphs},
  author={Money, Rohan T and Krishnan, Joshin P and Beferull-Lozano, Baltasar},
  journal={IEEE Transactions on Signal Processing},
  volume={71},
  pages={2027--2042},
  year={2023},
  publisher={IEEE}
}

@incollection{slavakis2014online,
  title={Online learning in reproducing kernel Hilbert spaces},
  author={Slavakis, Konstantinos and Bouboulis, Pantelis and Theodoridis, Sergios},
  booktitle={Academic Press Library in Signal Processing},
  volume={1},
  pages={883--987},
  year={2014},
  publisher={Elsevier}
}

@article{kingma2014adam,
  title={Adam: A method for stochastic optimization},
  author={Kingma, Diederik P and Ba, Jimmy},
  journal={arXiv preprint arXiv:1412.6980},
  year={2014}
}

@article{paszke2019pytorch,
  title={Pytorch: An imperative style, high-performance deep learning library},
  author={Paszke, Adam and Gross, Sam and Massa, Francisco and Lerer, Adam and Bradbury, James and Chanan, Gregory and Killeen, Trevor and Lin, Zeming and Gimelshein, Natalia and Antiga, Luca and others},
  journal={Advances in neural information processing systems},
  volume={32},
  year={2019}
}

@article{ruiz2023online,
  title={An online multiple kernel parallelizable learning scheme},
  author={Ruiz-Moreno, Emilio and Beferull-Lozano, Baltasar},
  journal={IEEE Signal Processing Letters},
  volume={31},
  pages={121--125},
  year={2023},
  publisher={IEEE}
}

@article{kou2013sparse,
  title={Sparse online warped Gaussian process for wind power probabilistic forecasting},
  author={Kou, Peng and Gao, Feng and Guan, Xiaohong},
  journal={Applied energy},
  volume={108},
  pages={410--428},
  year={2013},
  publisher={Elsevier}
}

@inproceedings{bui2016deep,
  title={Deep Gaussian processes for regression using approximate expectation propagation},
  author={Bui, Thang and Hern{\'a}ndez-Lobato, Daniel and Hernandez-Lobato, Jose and Li, Yingzhen and Turner, Richard},
  booktitle={International conference on machine learning},
  pages={1472--1481},
  year={2016},
  organization={PMLR}
}

\newpage
\onecolumn
\begin{center}
\Huge
    Supplementary material for ``Online Gradient Computation for Warping Gaussian Process Transformations''
\end{center}
\setcounter{equation}{0}
\setcounter{section}{0}
\renewcommand{\thesection}{S\arabic{section}} 
\renewcommand{\theequation}{s\arabic{equation}}
\titleformat*{\section}{\raggedright\Large\bfseries}
\medskip

\section{Proof of Theorem \ref{thm:recursive}}
\label{sup:proof}

The NLL in \eqref{eq:nll} corresponds to
\begin{subequations}
\begin{align}
    L_n(\bm{\theta}) &= -\log p(\bm{y}_n | \bm{\theta}) \\
    &= -\log q(\bm{g}(\bm{y}_n ; \bm{\theta})) - \log \prod^n_{i=1} \frac{\partial g(y_i;\bm{\theta})}{\partial y_i} \\
    &= \frac{1}{2} (2\pi)^n \left| \bm{\Omega}_n^{-1} \right| + \frac{1}{2} \bm{g}(\bm{y}_n ; \bm{\theta})^\top \, \bm{\Omega}_n \, \bm{g}(\bm{y}_n ; \bm{\theta}) - \sum^n_{i=1} \log \left( \frac{\partial g(y_i;\bm{\theta})}{\partial y_i} \right) .
\end{align}
\end{subequations}
Its gradient with respect to the warping parameters $\bm{\theta}$ is thus,
\begin{subequations}
\begin{align}
    \gradz L_n(\bm{\theta}) &= 
    \cancel{\gradz\frac{1}{2}(2\pi)^n \left| \bm{\Omega}_n^{-1} \right|} + \frac{1}{2}\gradz^\top \bm{g}(\bm{y}_n ; \bm{\theta}) \, \bm{\Omega}_n \, \bm{g}(\bm{y}_n ; \bm{\theta}) - \sum^n_{i=1} \gradz \log \left( \frac{\partial g(y_i;\bm{\theta})}{\partial y_i} \right) \\
    &= \dot{\bm{Z}}_n \bm{\Omega}_n \bm{z}_n - \sum^n_{i=1} \gradz \log \left( \frac{\partial z_i}{\partial y_i} \right) , \label{seq:grad_nll}
\end{align}
\end{subequations}
where 
\begin{subequations}
\begin{align}
    \bm{z}_n &= \bm{g}(\bm{y}_n;\bm{\theta}) = \begin{bmatrix} z_1 \\ \vdots \\ z_n\end{bmatrix} \in \mathbb{R}^n ,\\
    \dot{\bm{z}}_i &= \gradz g(y_i ; \bm{\theta}) \in \mathbb{R}^r, \text{ for all } i\in\{1,2,\dots,n\}, \text{ and } \\
    \dot{\bm{Z}}_n &= \gradz^\top \bm{g}(\bm{y}_n;\bm{\theta}) = \begin{bmatrix} \dot{\bm{z}}_1 & \dot{\bm{z}}_2 & \dots & \dot{\bm{z}}_n \end{bmatrix} = \begin{bmatrix} \dot{\bm{Z}}_{n-1} & \dot{\bm{z}}_n \end{bmatrix} \in \mathbb{R}^{r\times n} .
\end{align}
\end{subequations}

The first term in \eqref{seq:grad_nll} can be expanded by using the recursive representation of $\bm{\Omega}_n$ described in \eqref{seq:recursive_Omega} as
\begin{subequations}
\begin{align}
    \dot{\bm{Z}}_n \bm{\Omega}_n \bm{z}_n &= 
    \begin{bmatrix} \dot{\bm{Z}}_{n-1} & \dot{\bm{z}}_n \end{bmatrix}
    \left(
    \begin{bmatrix} \bm{\Omega}_{n-1} & \bm{0}_{n-1} \\ \bm{0}_{n-1}^\top & 0 \end{bmatrix}
    + \frac{1}{s_n}
    \begin{bmatrix} \bm{\omega}_n \\ -1 \end{bmatrix}
    \begin{bmatrix} \bm{\omega}_n^\top & -1 \end{bmatrix}
    \right)
    \begin{bmatrix} \bm{z}_{n-1} \\ z_n \end{bmatrix} \\
    &= \dot{\bm{Z}}_{n-1} \bm{\Omega}_{n-1} \bm{z}_{n-1} + \frac{1}{s_n} \left( \dot{\bm{Z}}_{n-1}\bm{\omega}_n - \dot{\bm{z}}_n \right) \left( \bm{\omega}_n^\top\bm{z}_{n-1} - z_n \right) \\
    &= \dot{\bm{Z}}_{n-1} \bm{\Omega}_{n-1} \bm{z}_{n-1} + \frac{e_n}{s_n}\bm{b}_n , \label{seq:quadratic_term_grad_nll}
\end{align}
\end{subequations}
where
\begin{subequations}
\begin{align}
    e_n &= z_n - \bm{\omega}_n^\top \bm{z}_{n-1} \in \mathbb{R} , \text{ and } \\
    \bm{b}_n &= \dot{\bm{z}}_n - \dot{\bm{Z}}_{n-1}\bm{\omega}_n \in \mathbb{R}^r . 
\end{align}
\end{subequations}
Note that $e_n$ already admits a recursive update rule since it is expressed in terms of the current latent target $z_n$ and its corresponding location $\bm{x}_n$, and the previous state terms $\bm{\alpha}_{n-1}$ and $\bm{\Omega}_{n-1}$.
Explicitly,
\begin{subequations}
\begin{align}
    e_n &= z_n - \bm{k}_{n-1}(\bm{x}_n)^\top \bm{\Omega}_{n-1} \bm{z}_{n-1} \\
    &= z_n - \bm{k}_{n-1}(\bm{x}_n)^\top \bm{\alpha}_{n-1} \\
    &= z_n - m_{n-1}(\bm{x}_n) .
\end{align}
\end{subequations}
Regarding $\bm{b}_n$, we maintain an additional state term $\bm{B}_n = \dot{\bm{Z}}_n \bm{\Omega}_n \in \mathbb{R}^{r\times n}$ to enable recursive updates.
Specifically,
\begin{subequations}
\begin{align}
    \bm{B}_n &= \begin{bmatrix} \dot{\bm{Z}}_{n-1} & \dot{\bm{z}}_n \end{bmatrix}
    \left( 
    \begin{bmatrix} \bm{\Omega}_{n-1} & \bm{0}_{n-1} \\ \bm{0}_{n-1}^\top & 0 \end{bmatrix}
    + \frac{1}{s_n}
    \begin{bmatrix} \bm{\omega}_n \\ -1 \end{bmatrix}
    \begin{bmatrix} \bm{\omega}_n^\top & -1 \end{bmatrix}
    \right) \\
    &= \begin{bmatrix} \dot{\bm{Z}}_{n-1}\bm{\Omega}_{n-1} & \bm{0}_r \end{bmatrix} + \frac{1}{s_n}
    \left(
    \dot{\bm{Z}}_{n-1}\bm{\omega}_n - \dot{\bm{z}}_n
    \right)
    \begin{bmatrix} \bm{\omega}_n^\top & -1 \end{bmatrix} \\
    &= \begin{bmatrix} \bm{B}_{n-1} & \bm{0}_r \end{bmatrix} - \frac{1}{s_n} 
    \left(
    \dot{\bm{z}}_n - \bm{B}_{n-1}\bm{k}_{n-1}(\bm{x}_n) 
    \right)
    \begin{bmatrix} \bm{\omega}_n^\top & -1 \end{bmatrix} ,
\end{align}
\end{subequations}
and accordingly,
\begin{subequations}
\begin{align}
    \bm{b}_n &= \dot{\bm{z}}_n - \dot{\bm{Z}}_{n-1} \bm{\Omega}_{n-1} \bm{k}_{n-1}(\bm{x}_n) \\
    &= \dot{\bm{z}}_n - \bm{B}_{n-1} \bm{k}_{n-1}(\bm{x}_n) .
\end{align}
\end{subequations}

Substituting \eqref{seq:quadratic_term_grad_nll} in \eqref{seq:grad_nll} we get
\begin{subequations}
\label{eq:recursive_nll_identic}
\begin{align}
    \gradz L_n(\bm{\theta}) &=  \dot{\bm{Z}}_{n-1} \bm{\Omega}_{n-1} \bm{z}_{n-1} + \frac{e_n}{s_n}\bm{b}_n - \sum^n_{i=1} \gradz \log \left( \frac{\partial z_i}{\partial y_i} \right) \\
    &= \dot{\bm{Z}}_{n-1} \bm{\Omega}_{n-1} \bm{z}_{n-1} - \sum^{n-1}_{i=1} \gradz \log \left( \frac{\partial z_i}{\partial y_i} \right) + \frac{e_n}{s_n}\bm{b}_n - \gradz \log \left( \frac{\partial z_n}{\partial y_n} \right) \\
    &= \gradz L_{n-1}(\bm{\theta}) + \frac{e_n}{s_n}\bm{b}_n - \gradz \log \left( \frac{\partial z_n}{\partial y_n} \right) .
\end{align}
\end{subequations}
Finally, by taking the gradient with respect to the warping parameters $\bm{\theta}$ in both sides of \eqref{eq:recursive_nll} and identifying terms with \eqref{eq:recursive_nll_identic} we obtain 
\begin{equation}
    \gradz \ell_n(\bm{\theta}) = \frac{e_n}{s_n}\bm{b}_n - \gradz \log \left( \frac{\partial z_n}{\partial y_n} \right) .
\end{equation}

\section{ALD sparsification}
\label{sup:ald_sparsification}

Consider a feature mapping $\phi:\mathcal{X} \to \mathcal{F}$ such that $\kappa(\bm{x},\bm{x}') = \langle \phi(\bm{x}),\phi(\bm{x}')\rangle_\mathcal{F}$ for all $\bm{x},\bm{x}'\in\mathcal{X}$, an ongoing trajectory of input locations $\bm{x}_1,\bm{x}_2,\dots,\bm{x}_{n-1}$, and a dictionary (sparse set) $\mathcal{D}_{n-1} \subseteq \mathcal{X}$ of representative input locations collected across that trajectory.

Every time a new input location $\bm{x}_n$ is presented, we check whether its feature representation $\phi(\bm{x}_n)$ is approximately linearly dependent on the feature representation of the previously collected input locations $\tilde{\bm{x}}_1, \tilde{\bm{x}}_2,\dots,\tilde{\bm{x}}_{D_{n-1}}$ in $\mathcal{D}_{n-1}$ where $D_n = |\mathcal{D}_n|$.
That is, we check whether the squared distance $\delta_n$ between $\phi(\bm{x}_n)$ and $\text{span}\{ \phi(\tilde{\bm{x}}_1),\phi(\tilde{\bm{x}}_2),\dots,\phi(\tilde{\bm{x}}_{D_{n-1}}) \}$ is less than or equal to a user-defined threshold $\nu \in \mathbb{R}_{\geq 0}$.
If it is, the dictionary remains the same, i.e., $\mathcal{D}_n = \mathcal{D}_{n-1}$.
If not, we update the dictionary to include the new input location, i.e., $\mathcal{D}_n = \mathcal{D}_{n-1}\cup\{\bm{x}_n\}$.

Explicitly, 
\begin{subequations}
\begin{align}
    \delta_n &= \underset{\bm{a}\in\mathbb{R}^{D_{n-1}}}{\text{ min }} 
    \left\Vert 
    \phi(\bm{x}_n) - \sum^{D_{n-1}}_{i=1} a_i \phi(\tilde{\bm{x}}_i) 
    \right\Vert^2_{\mathcal{F}} \\
    &= \underset{\bm{a}\in\mathbb{R}^{D_{n-1}}}{\text{ min }}
    \left\langle
    \phi(\bm{x}_n) - \sum^{D_{n-1}}_{i=1} a_i \phi(\tilde{\bm{x}}_i) , \phi(\bm{x}_n) - \sum^{D_{n-1}}_{j=1} a_j \phi(\tilde{\bm{x}}_j)
    \right\rangle_{\mathcal{F}} \\
    &= \underset{\bm{a}\in\mathbb{R}^{D_{n-1}}}{\text{ min }} 
    \left\langle \phi(\bm{x}_n) , \phi(\bm{x}_n) \right\rangle_{\mathcal{F}}
    -2 \sum^{D_{n-1}}_{i=1} a_i \left\langle \phi(\bm{x}_n) , \phi(\tilde{\bm{x}}_i) \right\rangle_{\mathcal{F}}
    + \sum^{D_{n-1}}_{i,j=1} a_i a_j \left\langle \phi(\tilde{\bm{x}}_i) , \phi(\tilde{\bm{x}}_j) \right\rangle_{\mathcal{F}} \\
    &= \underset{\bm{a}\in\mathbb{R}^{D_{n-1}}}{\text{ min }}
    \kappa(\bm{x}_n,\bm{x}_n) - 2\sum^{D_{n-1}}_{i=1} a_i \kappa(\bm{x}_n,\tilde{\bm{x}}_i) + \sum^{D_{n-1}}_{i,j=1} a_i a_j \kappa(\tilde{\bm{x}}_i,\tilde{\bm{x}}_j) \\
    &= \underset{\bm{a}\in\mathbb{R}^{D_{n-1}}}{\text{ min }} 
    \kappa(\bm{x}_n,\bm{x}_n) - 2\tilde{\bm{k}}_{n-1}(\bm{x}_n)^\top \bm{a} + \bm{a}^\top \tilde{\bm{K}}_{n-1} \bm{a} ,
\end{align}
\end{subequations}
where $\tilde{\bm{K}}_n\in\mathbb{R}^{D_n\times D_n}$ and $\tilde{\bm{k}}_{n}(\bm{x})\in\mathbb{R}^{D_n}$ are constructed as $[ \tilde{\bm{K}}_n ]_{i,j} = \kappa(\tilde{\bm{x}}_i,\tilde{\bm{x}}_j)$ and $[ \tilde{\bm{k}}_n(\bm{x}) ]_i = \kappa(\tilde{\bm{x}}_i,\bm{x})$ for all $\tilde{\bm{x}}_i,\tilde{\bm{x}}_j\in\mathcal{D}_n$ and $\bm{x}\in\mathcal{X}$.
Thus, finding $\delta_n$ consists of solving a convex quadratic optimization problem with a closed-form solution
\begin{subequations}
\label{eq:ald_solution}
\begin{align}
    \delta_n &= \kappa(\bm{x}_n,\bm{x}_n) - \tilde{\bm{k}}_{n-1}(\bm{x}_n)^\top\hat{\bm{a}}_n , \text{ where } \\
    \hat{\bm{a}}_n &= \tilde{\bm{K}}^{-1}_{n-1} \tilde{\bm{k}}_{n-1}(\bm{x_n}) . \label{seq:ald_a_hat}
\end{align}
\end{subequations}

From here, if $\bm{x}_n$ is included in $\mathcal{D}_n$ (meaning that $\delta_n > \nu$), we set the approximation coefficients as $\bm{a}_n = [0 , \dots , 0 , 1]^\top \in \{0\}^{D_n - 1} \times \{1\} \subset \mathbb{R}^{D_n}$ since $\phi(\bm{x}_n)$ can be exactly represented by itself.
If not (meaning that $\delta_n \leq \nu$), we set $\bm{a}_n = \hat{\bm{a}}_n$ with $\hat{\bm{a}}_n$ as in \eqref{seq:ald_a_hat}.

Let us construct the matrix $\bm{\Phi}_n = [\phi(\bm{x}_1), \phi(\bm{x}_2),\dots,\phi(\bm{x}_n)]$ of size $\text{dim}(\mathcal{F})\times n$.
By the sparsification procedure, we know that every $i$th feature $\phi(\bm{x}_i)$ is approximately linearly dependent on the features of the input locations in $\mathcal{D}_i$, i.e., $\phi(\bm{x}_i) \simeq \sum^{D_i}_{j=1} [\bm{a}_i]_j \phi(\tilde{\bm{x}}_j)$.
Based on this, we can approximate $\bm{\Phi}_n \simeq \tilde{\bm{\Phi}}_n \bm{A}_n^\top$ where $\tilde{\bm{\Phi}}_n = [\phi(\tilde{\bm{x}}_1),\phi(\tilde{\bm{x}}_2),\dots,\phi(\tilde{\bm{x}}_{D_n})]$ is a matrix of size $\text{dim}(\mathcal{F})\times D_n$, and $\bm{A}_n$ is a matrix of size $n\times D_n$ for which each $i$th row is constructed by padding zeros to $\bm{a}_i$ until dimension $D_n$.
On the other hand, we know that $\bm{K}_n = \bm{\Phi}_n^\top\bm{\Phi}_n$ by construction.
As a result, we can approximate
\begin{subequations}
\label{eq:tilde_K_matrix}
\begin{align}
    \bm{K}_n &= \bm{\Phi}_n^\top\bm{\Phi}_n \\
    &\simeq \bm{A}_n \tilde{\bm{\Phi}}_n^\top \tilde{\bm{\Phi}}_n \bm{A}_n^\top \\
    &= \bm{A}_n \tilde{\bm{K}}_n  \bm{A}_n^\top .
\end{align}
\end{subequations}
Similarly, 
\begin{subequations}
\label{eq:tilde_k_vector}
\begin{align}
    \bm{k}_n(\bm{x}) &= 
    \begin{bmatrix} \phi(\bm{x}_1)^\top\phi(\bm{x}) & \phi(\bm{x}_2)^\top\phi(\bm{x}) & \cdots & \phi(\bm{x}_n)^\top\phi(\bm{x})\end{bmatrix} = \bm{\Phi}^\top_n \phi(\bm{x}) \\
    &\simeq \left( \tilde{\bm{\Phi}}_n \bm{A}_n^\top \right)^\top \phi(\bm{x}) = \bm{A}_n \tilde{\bm{\Phi}}_n^\top \phi(\bm{x}) \\
    &= \bm{A}_n \tilde{\bm{k}}_n(\bm{x}) .
\end{align}
\end{subequations}

By extension, plugging \eqref{eq:tilde_K_matrix} and \eqref{eq:tilde_k_vector} into \eqref{eq:GP_moments} yields the following approximation of the posterior moments
\begin{subequations}
\label{eq:approximated_GP_moments}
\begin{align}
    m_n(\bm{x}) &\simeq \tilde{m}_n({\bm{x}}) = \tilde{\bm{k}}_n(\bm{x})^\top \tilde{\bm{\alpha}}_n , \text{ and } \\
    v_n(\bm{x}) &\simeq \tilde{v}_n(\bm{x}) = \kappa(\bm{x},\bm{x}) - \tilde{\bm{k}}_n(\bm{x})^\top \tilde{\bm{C}}_n \tilde{\bm{k}}_n(\bm{x}) ,
\end{align}
\end{subequations}
where
\begin{subequations}
\begin{align}
    \tilde{\bm{\alpha}}_n &= \bm{A}_n^\top \tilde{\bm{\Omega}}_n \bm{y}_n \in \mathbb{R}^{D_n}, \\
    \tilde{\bm{\Omega}}_n &= \left( \bm{A}_n \tilde{\bm{K}}_n  \bm{A}_n^\top + \bm{\Sigma}_n \right)^{-1} \in \mathbb{R}^{n\times n} , \text{ and } \\
    \tilde{\bm{C}}_n &= \bm{A}^\top_n \tilde{\bm{\Omega}}_n \bm{A}_n \in \mathbb{R}^{D_n \times D_n} .
\end{align}
\end{subequations}

\section{ALD sparse and recursive Gaussian process updates}
\label{sup:sparse-recursive_gp_updates}

The goal is updating the approximated posterior moments in \eqref{eq:approximated_GP_moments} from the current observation $y_n$, input location $\bm{x}_n$, and the previous state variables $\tilde{\bm{\alpha}}_{n-1}$ and $\tilde{\bm{C}}_{n-1}$.
Recall from Sec. \ref{sup:ald_sparsification} that every $n$th input location $\bm{x}_n$ may either be left out of the dictionary, in which case $\mathcal{D}_n = \mathcal{D}_{n-1}$, or added to it, in which case $\mathcal{D}_n = \mathcal{D}_{n-1}\cup\{\bm{x}_n\}$.
Depending on the case, the recursive update will vary accordingly.

\subsection{Case \texorpdfstring{$\mathcal{D}_n = \mathcal{D}_{n-1}$ (meaning that $\delta_n \leq \nu$)}{discarded}}

Since the dictionary remains unchanged 
\begin{subequations}
\begin{align}
    \tilde{\bm{K}}_n &= \tilde{\bm{K}}_{n-1} \in \mathbb{R}^{D_n \times D_n} , \\
    \bm{A}_n &= \begin{bmatrix} \bm{A}_{n-1} \\ \bm{a}_n^\top \end{bmatrix} \in \mathbb{R}^{n\times D_n} , \text{ and } \\
    \bm{a}_n &=  \tilde{\bm{K}}^{-1}_{n-1}\tilde{\bm{k}}_{n-1}(\bm{x}_n) .
\end{align}
\end{subequations}

Then,
\begin{subequations}
\begin{align}
    \tilde{\bm{\Omega}}_n^{-1} &= 
    \begin{bmatrix} \bm{A}_{n-1} \\ \bm{a}_n^\top \end{bmatrix} 
    \tilde{\bm{K}}_{n-1} 
    \begin{bmatrix} \bm{A}_{n-1}^\top & \bm{a}_n \end{bmatrix} + 
    \begin{bmatrix} \sigma^2\bm{I}_{n-1} & \bm{0}_{n-1} \\ \bm{0}_{n-1} & \sigma^2\end{bmatrix} \\
    &= 
    \begin{bmatrix}
        \bm{A}_{n-1} \tilde{\bm{K}}_{n-1} \bm{A}_{n-1}^\top & \bm{A}_{n-1}\tilde{\bm{K}}_{n-1} \bm{a}_n \\ \bm{a}_n^\top \tilde{\bm{K}}_{n-1} \bm{A}^\top_{n-1} & \bm{a}_n^\top \tilde{\bm{K}}_{n-1} \bm{a}_n
    \end{bmatrix} +
    \begin{bmatrix} \sigma^2\bm{I}_{n-1} & \bm{0}_{n-1} \\ \bm{0}_{n-1} & \sigma^2 \end{bmatrix} \\
    &= 
    \begin{bmatrix}
        \tilde{\bm{\Omega}}_n^{-1} & \bm{A}_{n-1}\tilde{\bm{K}}_{n-1} \bm{a}_n \\ \left( \bm{A}_{n-1}\tilde{\bm{K}}_{n-1} \bm{a}_n \right)^\top & \bm{a}_n^\top \tilde{\bm{K}}_{n-1} \bm{a}_n + \sigma^2 
    \end{bmatrix} .
\end{align}
\end{subequations}
Notice that,
\begin{subequations}
\begin{align}
    \bm{A}_{n-1} \tilde{\bm{K}}_{n-1} \bm{a}_n &= \bm{A}_{n-1} \tilde{\bm{K}}_{n-1} \tilde{\bm{K}}_{n-1}^{-1} \tilde{\bm{k}}_{n-1}(\bm{x}_n) \\
    &= \bm{A}_{n-1} \tilde{\bm{k}}_{n-1}(\bm{x}_n) , \text{ and } \\
    \bm{a}^\top_n \tilde{\bm{K}}_{n-1} \bm{a}_n &= \tilde{\bm{k}}_{n-1}(\bm{x}_n)^\top \tilde{\bm{K}}_{n-1}^{-1} \tilde{\bm{K}}_{n-1} \bm{a}_n \\
    &= \tilde{\bm{k}}_{n-1}(\bm{x}_n)^\top \bm{a}_n .
\end{align}
\end{subequations}
Therefore,
\begin{equation}
    \tilde{\bm{\Omega}}_n^{-1} = 
    \begin{bmatrix}
        \tilde{\bm{\Omega}}_{n-1}^{-1} & \bm{A}_{n-1} \tilde{\bm{k}}_{n-1}(\bm{x}_n) \\
        \left( \bm{A}_{n-1} \tilde{\bm{k}}_{n-1}(\bm{x}_n) \right)^\top & \tilde{\bm{k}}_{n-1}(\bm{x}_n)^\top\bm{a}_n + \sigma^2
    \end{bmatrix} ,
\end{equation}
and applying the partitioned (symmetric and positive semi-definite) matrix inverse formula \cite[Ch. 0.7]{horn2012matrix} we get 
\begin{equation}
    \tilde{\bm{\Omega}}_n =
    \begin{bmatrix}
        \tilde{\bm{\Omega}}_{n-1} & \bm{0}_{n-1} \\ \bm{0}_{n-1}^\top & 0
    \end{bmatrix}
    + \frac{1}{\tilde{s}_n}
    \begin{bmatrix} \tilde{\bm{\omega}}_n \\ -1 \end{bmatrix}
    \begin{bmatrix} \tilde{\bm{\omega}}_n^\top & -1 \end{bmatrix} ,
\end{equation}
where
\begin{subequations}
\begin{align}
    \tilde{\bm{\omega}}_n &= \tilde{\bm{\Omega}}_{n-1}\bm{A}_{n-1}\tilde{\bm{k}}_{n-1}(\bm{x}_n) \in \mathbb{R}^{n-1} , \text{ and  } \\
    \tilde{s}_n &=  \tilde{k}_{n-1}(\bm{x}_n)^\top \bm{a}_n + \sigma^2 - \tilde{\bm{\omega}}_n^\top \bm{A}_{n-1} \tilde{\bm{k}}_{n-1}(\bm{x}_n) \in \mathbb{R} .
\end{align}
\end{subequations}

Now,
\begin{subequations}
\begin{align}
    \tilde{\bm{C}}_n &= 
    \begin{bmatrix} \bm{A}_{n-1}^\top & \bm{a}_n \end{bmatrix}
    \left(
    \begin{bmatrix} \tilde{\bm{\Omega}}_{n-1} & \bm{0}_{n-1} \\ \bm{0}_{n-1}^\top & 0 \end{bmatrix} +
    \frac{1}{\tilde{s}_n}
    \begin{bmatrix} \tilde{\bm{\omega}}_n \\ -1 \end{bmatrix}
    \begin{bmatrix} \tilde{\bm{\omega}}_n^\top & -1 \end{bmatrix}
    \right)
    \begin{bmatrix} \bm{A}_{n-1} \\ \bm{a}_n^\top \end{bmatrix} \\
    &= \bm{A}_{n-1}^\top \tilde{\bm{\Omega}}_{n-1} \bm{A}_{n-1} + \frac{1}{\tilde{s}_n}
    \left( \bm{A}_{n-1}^\top \tilde{\bm{\omega}}_n - \bm{a}_n \right)
    \left( \tilde{\bm{\omega}}_n^\top \bm{A}_{n-1} - \bm{a}_n^\top \right) \\
    &= \tilde{\bm{C}}_{n-1} + \frac{1}{\tilde{s}_n} \tilde{\bm{c}}_n \tilde{\bm{c}}_n^\top ,
\end{align}
\end{subequations}
where
\begin{equation}
    \tilde{\bm{c}}_n = \bm{a}_n - \bm{A}_{n-1}^\top \tilde{\bm{w}}_n \in \mathbb{R}^{D_n} .
\end{equation}

In the same way,
\begin{subequations}
\begin{align}
    \tilde{\bm{\alpha}}_n &= 
    \begin{bmatrix} \bm{A}_{n-1}^\top & \bm{a}_n \end{bmatrix}
    \left(
    \begin{bmatrix} \tilde{\bm{\Omega}}_{n-1} & \bm{0}_{n-1} \\ \bm{0}_{n-1}^\top & 0 \end{bmatrix} +
    \frac{1}{\tilde{s}_n}
    \begin{bmatrix} \tilde{\bm{\omega}}_n \\ -1 \end{bmatrix}
    \begin{bmatrix} \tilde{\bm{\omega}}_n^\top & -1 \end{bmatrix}
    \right)
    \begin{bmatrix} \bm{y}_{n-1} \\ y_n \end{bmatrix} \\
    &= \bm{A}_{n-1}^\top \tilde{\bm{\Omega}}_{n-1} \bm{y}_{n-1} + \frac{1}{\tilde{s}_n}
    \left( \bm{A}_{n-1}^\top \tilde{\bm{\omega}}_n - \bm{a}_n \right)
    \left( \tilde{\bm{\omega}}_n^\top \bm{y}_{n-1} - y_n \right) \\
    &= \tilde{\bm{\alpha}}_{n-1} + \frac{\tilde{e}_n}{\tilde{s}_n} \tilde{\bm{c}}_n ,
\end{align}
\end{subequations}
where
\begin{equation}
    \tilde{e}_n = y_n - \tilde{\bm{\omega}}_n^\top \bm{y}_{n-1} \in \mathbb{R}.
\end{equation}

Finally,
\begin{subequations}
\begin{align}
    \tilde{\bm{c}}_n &= \bm{a}_n - \bm{A}_{n-1}^\top \tilde{\bm{\Omega}}_{n-1} \bm{A}_{n-1} \tilde{\bm{k}}_{n-1}(\bm{x}_n) \\
    &= \bm{a}_n - \tilde{\bm{C}}_{n-1} \tilde{\bm{k}}_{n-1}(\bm{x}_n) , \\
    \tilde{e}_n &= y_n - \tilde{\bm{k}}_{n-1}(\bm{x}_n)^\top \bm{A}_{n-1}^\top \tilde{\bm{\Omega}}_{n-1} \bm{y}_{n-1} \\
    &= y_n - \tilde{\bm{k}}_{n-1}(\bm{x}_n)^\top \tilde{\bm{\alpha}}_{n-1} , \text{ and }\\
    \tilde{s}_n &= \tilde{\bm{k}}_{n-1}(\bm{x}_n)^\top \bm{a}_n + \sigma^2 - \tilde{\bm{k}}_{n-1}(\bm{x}_n) \bm{A}_{n-1}^\top \tilde{\bm{\Omega}}_{n-1} \bm{A}_{n-1} \tilde{\bm{k}}_{n-1}(\bm{x}_n) \\
    &= \tilde{\bm{k}}_{n-1}(\bm{x}_n)^\top \bm{a}_n + \sigma^2 - \tilde{\bm{k}}_{n-1}(\bm{x}_n)^\top \tilde{\bm{C}}_{n-1} \tilde{\bm{k}}_{n-1}(\bm{x}_n) .
\end{align}
\end{subequations}

\subsection{Case \texorpdfstring{$\mathcal{D}_n = \mathcal{D}_{n-1}\cup\{\bm{x}_n\}$ (meaning that $\delta_n > \nu$)}{added}}

In this case, 
\begin{subequations}
\begin{align}
    \tilde{\bm{K}}_n &= 
    \begin{bmatrix} \tilde{\bm{K}}_{n-1} & \tilde{\bm{k}}_{n-1}(\bm{x}_n) \\ \tilde{\bm{k}}_{n-1}(\bm{x}_n)^\top & \kappa(\bm{x}_n,\bm{x}_n)  \end{bmatrix} \in \mathbb{R}^{D_n \times D_n} , \label{seq:tilde_K_matrix} \\
    \bm{A}_n &= 
    \begin{bmatrix} 
    \begin{bmatrix} \bm{A}_{n-1} & \bm{0}_{n-1} \end{bmatrix} \\ 
    \bm{a}_n
    \end{bmatrix} \in \mathbb{R}^{n\times D_n} , \text{ and } \\
    \bm{a}_n &= \begin{bmatrix} 0 & \cdots & 0 & 1 \end{bmatrix} \in \{0\}^{D_n - 1} \times \{1\} \subset \mathbb{R}^{D_n} . 
\end{align}
\end{subequations}

Then, 
\begin{subequations}
\begin{align}
    \tilde{\bm{\Omega}}_n^{-1} &= 
    \begin{bmatrix} 
        \begin{bmatrix} \bm{A}_{n-1} & \bm{0}_{n-1} \end{bmatrix} \\ 
        \bm{a}_n
    \end{bmatrix}
    \begin{bmatrix} 
        \tilde{\bm{K}}_{n-1} & \tilde{\bm{k}}_{n-1}(\bm{x}_n) \\ \tilde{\bm{k}}_{n-1}(\bm{x}_n)^\top & \kappa(\bm{x}_n,\bm{x}_n)  
    \end{bmatrix}
    \begin{bmatrix}
        \begin{bmatrix} \bm{A}_{n-1}^\top \\ \bm{0}_{n-1}^\top \end{bmatrix} &
        \bm{a}_n
    \end{bmatrix} + 
    \begin{bmatrix}
        \sigma^2\bm{I}_{n-1} & \bm{0}_{n-1} \\ \bm{0}_{n-1}^\top & \sigma^2
    \end{bmatrix} \\
    &= 
    \begin{bmatrix}
        \bm{A}_{n-1} \tilde{\bm{K}}_{n-1} & \bm{A}_{n-1}\tilde{\bm{k}}_{n-1}(\bm{x}_n) \\
        \tilde{\bm{k}}_{n-1}(\bm{x}_n)^\top & \kappa(\bm{x}_n,\bm{x}_n)
    \end{bmatrix}
    \begin{bmatrix}
        \begin{bmatrix} \bm{A}_{n-1}^\top \\ \bm{0}_{n-1}^\top \end{bmatrix} &
        \bm{a}_n
    \end{bmatrix} + 
    \begin{bmatrix}
        \sigma^2\bm{I}_{n-1} & \bm{0}_{n-1} \\ \bm{0}_{n-1}^\top & \sigma^2
    \end{bmatrix} \\
    &= 
    \begin{bmatrix}
        \bm{A}_{n-1} \tilde{\bm{K}}_{n-1} \bm{A}_{n-1}^\top + \sigma^2\bm{I}_{n-1} & \bm{A}_{n-1} \tilde{\bm{k}}_{n-1}(\bm{x}_n) \\
        \tilde{\bm{k}}_{n-1}(\bm{x}_n)^\top \bm{A}_{n-1}^\top & \kappa(\bm{x}_n,\bm{x}_n) + \sigma^2
    \end{bmatrix} \\
    &= 
    \begin{bmatrix}
        \tilde{\bm{\Omega}}_{n-1}^{-1} & \bm{A}_{n-1} \tilde{\bm{k}}_{n-1}(\bm{x}_n) \\
        \left( \bm{A}_{n-1} \tilde{\bm{k}}_{n-1}(\bm{x}_n) \right)^\top & \kappa(\bm{x}_n,\bm{x}_n) + \sigma^2
    \end{bmatrix} ,
\end{align}
\end{subequations}
and applying the partitioned (symmetric and positive semi-definite) matrix inverse formula \cite[Ch. 0.7]{horn2012matrix} we get 
\begin{equation}
    \tilde{\bm{\Omega}}_n =
    \begin{bmatrix}
        \tilde{\bm{\Omega}}_{n-1} & \bm{0}_{n-1} \\ \bm{0}_{n-1}^\top & 0
    \end{bmatrix}
    + \frac{1}{\tilde{s}_n}
    \begin{bmatrix} \tilde{\bm{\omega}}_n \\ -1 \end{bmatrix}
    \begin{bmatrix} \tilde{\bm{\omega}}_n^\top & -1 \end{bmatrix} ,
\end{equation}
where
\begin{subequations}
\begin{align}
    \tilde{\bm{\omega}}_n &= \tilde{\bm{\Omega}}_{n-1}\bm{A}_{n-1}\tilde{\bm{k}}_{n-1}(\bm{x}_n) \in \mathbb{R}^{n-1} , \text{ and  } \\
    \tilde{s}_n &=  \kappa(\bm{x}_n,\bm{x}_n) + \sigma^2 - \tilde{\bm{\omega}}_n^\top \bm{A}_{n-1} \tilde{\bm{k}}_{n-1}(\bm{x}_n) \in \mathbb{R} .
\end{align}
\end{subequations}

From here, we are able to expand
\begin{subequations}
\begin{align}
    \tilde{\bm{C}}_n &=
    \begin{bmatrix}
        \begin{bmatrix} \bm{A}_{n-1}^\top \\ \bm{0}_{n-1}^\top \end{bmatrix} &
        \bm{a}_n
    \end{bmatrix}
    \left(
    \begin{bmatrix} \tilde{\bm{\Omega}}_{n-1} & \bm{0}_{n-1} \\ \bm{0}_{n-1}^\top & 0 \end{bmatrix} +
    \frac{1}{\tilde{s}_n}
    \begin{bmatrix} \tilde{\bm{\omega}}_n \\ -1 \end{bmatrix}
    \begin{bmatrix} \tilde{\bm{\omega}}_n^\top & -1 \end{bmatrix}
    \right)
    \begin{bmatrix} 
        \begin{bmatrix} \bm{A}_{n-1} & \bm{0}_{n-1} \end{bmatrix} \\ 
        \bm{a}_n
    \end{bmatrix} \\
    &= 
    \begin{bmatrix}
        \bm{A}_{n-1}^\top \\ \bm{0}_{n-1}^\top 
    \end{bmatrix}
    \tilde{\bm{\Omega}}_{n-1}
    \begin{bmatrix}
        \bm{A}_{n-1} & \bm{0}_{n-1} 
    \end{bmatrix} 
    + \frac{1}{\tilde{s}_n} 
    \left(
    \begin{bmatrix} \bm{A}_{n-1}^\top \\ \bm{0}_{n-1}^\top \end{bmatrix}
    \tilde{\bm{\omega}}_n - \bm{a}_n
    \right)
    \left( 
    \tilde{\bm{\omega}}_n^\top 
    \begin{bmatrix} \bm{A}_{n-1} & \bm{0}_{n-1} \end{bmatrix} - \bm{a}_n^\top
    \right) \\
    &= 
    \begin{bmatrix}
        \tilde{\bm{C}}_{n-1} & \bm{0}_{D_n-1} \\ \bm{0}_{D_n-1}^\top & 0
    \end{bmatrix}
    + \frac{1}{\tilde{s}_n} \tilde{\bm{c}}_n \tilde{\bm{c}}_n^\top
\end{align}
\end{subequations}
where
\begin{equation}
    \tilde{\bm{c}}_n = \bm{a}_n - 
    \begin{bmatrix} \bm{A}_{n-1}^\top \\ \bm{0}_{n-1}^\top \end{bmatrix} \tilde{\bm{w}}_n \in \mathbb{R}^{D_n} .
\end{equation}

Similarly,
\begin{subequations}
\begin{align}
    \tilde{\bm{\alpha}}_n &= 
    \begin{bmatrix}
        \begin{bmatrix} \bm{A}_{n-1}^\top \\ \bm{0}_{n-1}^\top \end{bmatrix} &
        \bm{a}_n
    \end{bmatrix}
    \left(
    \begin{bmatrix} \tilde{\bm{\Omega}}_{n-1} & \bm{0}_{n-1} \\ \bm{0}_{n-1}^\top & 0 \end{bmatrix} +
    \frac{1}{\tilde{s}_n}
    \begin{bmatrix} \tilde{\bm{\omega}}_n \\ -1 \end{bmatrix}
    \begin{bmatrix} \tilde{\bm{\omega}}_n^\top & -1 \end{bmatrix}
    \right)
    \begin{bmatrix} \bm{y}_{n-1} \\ y_n \end{bmatrix} \\
    &= 
    \begin{bmatrix} \bm{A}_{n-1}^\top \\ \bm{0}_{n-1}^\top  \end{bmatrix}
    \tilde{\bm{\Omega}}_{n-1} \bm{y}_{n-1} + \frac{1}{\tilde{s}_n}
    \left(
    \begin{bmatrix} \bm{A}_{n-1}^\top \\ \bm{0}_{n-1}^\top  \end{bmatrix} 
    \tilde{\bm{\omega}}_n - \bm{a}_n
    \right)
    \left(
    \tilde{\bm{\omega}}_n^\top \bm{y}_{n-1} - y_n
    \right) \\
    &= 
    \begin{bmatrix} \tilde{\bm{\alpha}}_{n-1} \\ 0 \end{bmatrix} 
    + \frac{\tilde{e}_n}{\tilde{s}_n} \tilde{\bm{c}}_n ,
\end{align}
\end{subequations}
where
\begin{equation}
    \tilde{e}_n = y_n - \tilde{\bm{\omega}}_n^\top \bm{y}_{n-1} \in \mathbb{R} .
\end{equation}

Finally,
\begin{subequations}
\begin{align}
    \tilde{\bm{c}}_n &= \bm{a}_n -
    \begin{bmatrix} \bm{A}_{n-1}^\top \tilde{\bm{\Omega}}_{n-1} \bm{A}_{n-1} \tilde{\bm{k}}_{n-1}(\bm{x}_n) \\ 0 \end{bmatrix} \\
    &= 
    \begin{bmatrix} - \tilde{\bm{C}}_{n-1} \tilde{\bm{k}}_{n-1}(\bm{x}_n) \\ 1 \end{bmatrix} , \\
    \tilde{e}_n &= y_n - \tilde{\bm{k}}_{n-1}(\bm{x}_n)^\top\bm{A}_{n-1}^\top \tilde{\bm{\Omega}}_{n-1} \bm{y}_{n-1} \\
    &= y_n - \tilde{\bm{k}}_{n-1}(\bm{x}_n)^\top \tilde{\bm{\alpha}}_{n-1} , \text{ and } \\
    \tilde{s}_n &= \kappa(\bm{x}_n,\bm{x}_n) + \sigma^2 - \tilde{\bm{k}}_{n-1}(\bm{x}_n)^\top \bm{A}_{n-1}^\top \tilde{\bm{\Omega}}_{n-1} \bm{A}_{n-1} \tilde{\bm{k}}_{n-1}(\bm{x}_n) \\
    &= \kappa(\bm{x}_n,\bm{x}_n) + \sigma^2 - \tilde{\bm{k}}_{n-1}(\bm{x}_n)^\top \tilde{\bm{C}}_{n-1} \tilde{\bm{k}}_{n-1}(\bm{x}_n) .
\end{align}
\end{subequations}

It is worth noting that by applying the partitioned matrix inverse formula \cite[Ch. 0.7]{horn2012matrix} on \eqref{seq:tilde_K_matrix} one can get
\begin{equation}
    \tilde{\bm{K}}_{n-1}^{-1} = \frac{1}{\delta_n} 
    \begin{bmatrix}
        \delta_n\tilde{\bm{K}}_{n-1}^{-1} + \hat{\bm{a}}_n \hat{\bm{a}}_n^\top & - \hat{\bm{a}}_n \\ -\hat{\bm{a}}_n^\top & 1
    \end{bmatrix} ,
\end{equation}
with $\delta_n$ and $\hat{\bm{a}}_n$ as in \eqref{eq:ald_solution}.

\subsection{Comments on the computational and memory complexity}
The per-step computational and memory complexity of Algorithm \ref{alg} is $\mathcal{O}(D_n^2 + rD_n)$, where $D_n$ is the ALD sparsification dictionary size, and $r$ is the number of warping parameters.

The following operations dominate this complexity:
\begin{itemize}
    \item Evaluating the kernel vector $\tilde{\bm{k}}_{n-1}(\bm{x}_n)$ is $\mathcal{O}(D_n)$, since it evaluates the kernel between $\bm{x}_n$ and each of the $D_n$ dictionary elements.
    \item Computing the ALD projection coefficients $\hat{\bm{a}}_n = \tilde{\bm{K}}^{-1}_{n-1}\tilde{\bm{k}}_{n-1}(\bm{x}_n)$ is $\mathcal{O}(D_n^2)$ since $\tilde{\bm{K}}^{-1}_n$ is updated via a rank-1 block inversion.
    \item Evaluating the warping transformation used in Sec. \ref{sec:experiment} $g(y_n;\bm{\theta})$ (and its gradient) is $\mathcal{O}(r)$ due to basic arithmetic and $\tanh$ operations.
    \item Computing the matrix-vector products $\tilde{\bm{B}}_{n-1}\tilde{\bm{k}}_{n-1}(\bm{x}_n)$ and $\tilde{\bm{C}}_{n-1}\tilde{\bm{k}}_{n-1}(\bm{x}_n)$ is $\mathcal{O}(rD_n)$ and $\mathcal{O}(D^2_n)$, respectively, due to the matrix and vector dimensions.
    \item Finally, updating $\tilde{\bm{C}}_n$ and $\tilde{\bm{B}}_n$ is $\mathcal{O}(D_n^2)$ and $\mathcal{O}(rD_n)$, respectively, due to the rank-1 outer products $\tilde{\bm{c}}_n\tilde{\bm{c}}_n^\top$ and $\tilde{\bm{b}}_n\tilde{\bm{c}}^\top_n$ in their update rules.
\end{itemize}

\end{document}